\documentclass{article}

\usepackage{aaai_arxiv}

\usepackage{times}
\usepackage{helvet}
\usepackage{courier}
\usepackage{url}
\usepackage{graphicx}
\newcommand{\citet}[1]{\citeauthor{#1} ̃\shortcite{#1}}
\newcommand{\citep}{\cite}

\usepackage{amsfonts}
\usepackage{amsmath}
\usepackage{algorithm}
\usepackage{algpseudocode}
\usepackage{algorithmicx}
\usepackage{graphicx}
\usepackage{multirow}
\usepackage{subcaption}
\usepackage{amsthm}

\DeclareMathOperator*{\argmin}{arg\,min}
\DeclareMathOperator{\Exp}{\mathbb{E}}
\DeclareMathOperator{\Var}{\mathbb{V}}
\newcommand{\norm}[1]{\left\lVert #1 \right\rVert}
\newcommand{\abs}[1]{\left| #1 \right|}
\newcommand{\E}[2][]{\Exp_{#1}\left[ #2 \right]}
\newcommand{\Tr}[1]{\text{Tr}\left( #1 \right)}
\newcommand{\V}[2][]{\Var_{#1}\left[ #2 \right]}

\newcommand{\diff}[2]{\frac{\partial #1}{\partial #2}}
\newcommand{\ddiff}[2]{\frac{\partial^2 #1}{\partial #2^2}}
\newtheorem{theorem}{Theorem}
\newtheorem{lemma}{Lemma}

\usepackage{color}
\newcounter{nbdrafts}
\makeatletter
\newcommand{\checknbdrafts}{
\ifnum \thenbdrafts > 0
\@latex@warning@no@line{*WARNING* The document contains \thenbdrafts \space draft note(s)}
\fi}

\makeatother
\title{Biased Importance Sampling for Deep Neural Network Training}

\author{
    Angelos Katharopoulos \quad Fran{\c{c}}ois Fleuret \\
    Idiap Research Institute, Martigny, Switzerland \\
    {\'E}cole Polytechique F{\'e}d{\'e}rale de Lausanne (EPFL), Lausanne, Switzerland \\
    \texttt{\{name.surname\}@idiap.ch}
}

\begin{document}

\maketitle

\begin{abstract}

Importance sampling has been successfully used to accelerate stochastic
optimization in many convex problems. However, the lack of an efficient way to
calculate the importance still hinders its application to Deep Learning.

In this paper, we show that the loss value can be used as an alternative
importance metric, and propose a way to efficiently approximate it for a deep
model, using a small model trained for that purpose in parallel.

This method allows in particular to utilize a biased gradient estimate that
implicitly optimizes a soft max-loss, and leads to better generalization
performance. While such method suffers from a prohibitively high variance of
the gradient estimate when using a standard stochastic optimizer, we show that
when it is combined with our sampling mechanism, it results in a reliable
procedure.

We showcase the generality of our method by testing it on both image
classification and language modeling tasks using deep convolutional and
recurrent neural networks. In particular, our method results in 30\% faster
training of a CNN for CIFAR10 than when using uniform sampling.

\end{abstract}

\section{Introduction}

The dramatic increase in available training data has made the use of Deep
Neural Networks feasible, which in turn has significantly improved the
state-of-the-art in many fields, in particular Computer Vision and Natural
Language Processing. However, due to the complexity of the resulting
optimization problem, computational cost is now the core issue in training
these large architectures.

When training such models, it appears to any practitioner that not all samples
are equally important; many of them are properly handled after a few epochs of
training, and most could be ignored at that point without impacting the final
model. To this end, we propose a novel importance sampling scheme that
accelerates the training, in theory, of any Neural Network architecture.

For convex optimization problems, many works
\cite{bordes2005fast,zhao2015stochastic,needell2014stochastic,canevet-et-al-2016,richtarik2013optimal}
have taken advantage of the difference in importance among the samples to
improve the convergence speed of stochastic optimization methods. However, only
recently, researchers \cite{alain2015variance,loshchilov2015online} have
shifted their focus on using importance sampling for training Deep Neural
Networks. Compared to these works, we propose an importance sampling scheme
based on the loss and show that our method can be used to improve the
convergence of various architectures on realistic text and image datasets, while
at the same time using a minimal number of hyperparameters.

\citet{zhao2015stochastic} prove that sampling with a distribution proportional
to the gradient norm of each sample, is optimal in minimizing the variance of
the gradients; thus resulting in a convergence speed-up for Stochastic Gradient
Descent (SGD). However, using the gradient norm requires computing second-order
quantities during the backward pass, which is computationally prohibitive. We
show, theoretically, that using the loss, we can construct an upper bound to
the gradient norm that is better than uniform. Nevertheless, computing the loss
still requires a full forward pass, hence using it directly on all the samples
remains intractable. In order to reduce the computational cost even more, we
propose to use the prediction of a small network, trained alongside the deep
model we want to eventually train, to predict an approximation of the
importance of the training samples. The complexity of this surrogate allows us
to modulate the cost / accuracy trade-off.

Finally, we also show a relationship between importance sampling and
maximum-loss minimization, which can be used to improve the
generalization ability of the trained Deep Neural Network.

In summary, the contributions of this work are:
\begin{itemize}
    \item We show that using the loss, we can construct a sampling distribution
        that reduces the variance of the gradients compared to uniform sampling
    \item The creation of a model able to approximate the loss for a
      low computational overhead
    \item The development of an online algorithm that minimizes a soft
      max-loss in the training set through importance sampling
\end{itemize}

\section{Related Work}

Importance sampling for convex problems has received significant attention over
the years. \citet{bordes2005fast} developed LASVM, which is an online algorithm
that uses importance sampling to train kernelized support vector machines.
Later \citet{needell2014stochastic} and more recently \citet{zhao2015stochastic}
developed more general importance sampling methods that improve the convergence
of Stochastic Gradient Descent. In particular, the latter has crucially
connected the convergence speed of SGD with the variance of the gradient
estimator and has shown that the target sampling distribution is the one that
minimizes this variance.

\citet{alain2015variance} are the first ones, to our knowledge, that
attempted to use importance sampling for training Deep Neural Networks. They
sample according to the exact gradient norm as computed by a cluster of GPU
workers. Even with a cluster of GPUs they have to constrain the networks that
they use to fully connected layers in order to be able to compute the gradient
norm in a reasonable time. Compared to their approach, the proposed method
achieves wall-clock time speed-up, for any architecture, requiring the same
resources as plain uniform sampling.

\citet{loshchilov2015online} looked towards the loss value to build a sampling
distribution for mini-batch creation for Deep Neural Networks. Their method
provides the Neural Network with samples that have been previously seen to have
high loss. The most important limitation of their work is the introduction of
several hard to choose hyperparameters that also impede the generalization of
their method to datasets other than MNIST. On the other hand, we conduct
experiments with more challenging image and text datasets and show that our
method generalizes well without the need to choose any hyperparameters.

\citet{canevet-et-al-2016} worked on improving the sampling procedure for
importance sampling. They imposed a prior tree structure on the weights,
and use a sampling procedure inspired by the Monte Carlo Tree Search algorithm,
that handles properly the exploration / exploitation dilemma and converges
asymptotically to the probability distribution of the weights.
However, this method fully relies on the existence of the tree structure,
that should reflect the regularity of the importance on the samples, which
is in itself a quite complicated embedding problem.

Finally, there is another class of methods related to importance sampling that
can be perceived as using an importance metric quite antithetical to most
common methods. Curriculum learning \cite{bengio2009curriculum} and its
evolution self-paced learning \cite{kumar2010self} present the classifier with
easy samples first (samples that are likely to have a small loss) and gradually
introduce harder and harder samples.

\section{Importance Sampling}

Importance sampling aims at increasing the convergence speed of SGD by reducing
the variance of the gradient estimates. In the following sections, we analyze
how this works and present an efficient procedure that can be used to train any
Deep Learning model. Moreover, we also show that our importance sampling method
has a close relation to maximum loss minimization, which can result in improved
generalization performance.

\subsection{Exact Importance Sampling}

Let $x_i$, $y_i$ be the $i$-th input-output pair from the training set, $\Psi(\cdot;
\theta)$ a Deep Learning model parameterized by the vector $\theta$, and
$L(\cdot, \cdot)$ the loss function to be minimized during training.
The goal of training is to find
\begin{equation}
\theta^* = \argmin_\theta \frac{1}{N} \sum_{i=1}^N L(\Psi(x_i; \theta), y_i)
\end{equation}
where $N$ corresponds to the number of examples in the training set. Using
Stochastic Gradient Descent with learning rate $\eta$, we iteratively update
the parameters of our model, between two consecutive iterations $t$ and $t+1$,
with
\begin{equation}
\theta_{t+1} = \theta_t - \eta \alpha_i \nabla_{\theta_t} L(\Psi(x_i; \theta_t), y_i)
\end{equation}
where $i$ is a discrete random variable sampled according to a distribution $P$
with probabilities $p_i$ and $\alpha_i$ is a sample weight. For instance, plain SGD with
uniform sampling is achieved with $\alpha_i = 1$ and $p_i = \frac{1}{N}$ for all $i$.

If we define the convergence speed $S$ of SGD as the reduction of the distance
of the parameter vector $\theta$ from the optimal parameter vector $\theta^*$
in two consecutive iterations $t$ and $t+1$
\begin{equation} \label{eq:convergence_speed_first_norms}
S = -\E[P]{\norm{\theta_{t+1} - \theta^*}_2^2 - \norm{\theta_{t} - \theta^*}_2^2},
\end{equation}
and if we have
\begin{equation}
\E[P]{\alpha_i \nabla_{\theta_t} L(\Psi(x_i; \theta_t), y_i)} =
    \nabla_{\theta_t} \frac{1}{N} \sum_{i=1}^N L(\Psi(x_i; \theta_t), y_i),
\end{equation}
and set $G_i = \alpha_i \nabla_{\theta_t}L(\Psi(x_i; \theta_t), y_i)$,
then we get (this is a different derivation of the result by \citet{wang2016accelerating})
\begin{equation} \label{eq:convergence_speed_first} 
\begin{aligned}
S
    & = -\E[P]{
        \left(\theta_{t+1} - \theta^*\right)^T \left(\theta_{t+1} - \theta^*\right) -
        \left(\theta_t - \theta^*\right)^T \left(\theta_t - \theta^*\right)
    } \\
    & = -\E[P]{
        \theta_{t+1}^T \theta_{t+1} - 2 \theta_{t+1} \theta^* -
        \theta_{t}^T \theta_{t} + 2 \theta_{t} \theta^*
    } \\
    & = -\E[P]{
        \left(\theta_t - \eta G_i \right)^T \left(\theta_t - \eta G_i \right) +
        2 \eta G_i^T \theta^* - \theta_t^T \theta_t
    } \\
    & = -\E[P]{
        -2 \eta \left(\theta_t - \theta^*\right) G_i + \eta^2 G_i^T G_i
    } \\
    & = 2 \eta \left(\theta_t - \theta^*\right) \E[P]{G_i} -
        \eta^2 \E[P]{G_i}^T \E[P]{G_i} - \\
    & \quad \, \eta^2 \Tr{\V[P]{G_i}}
\end{aligned}
\end{equation}

From the last expression, we observe that it is possible to gain a speedup by
sampling from the distribution that minimizes $\Tr{\V[P]{G_i}}$.
\citet{alain2015variance} and \citet{zhao2015stochastic} show that this
distribution has probabilities $p_i \propto \norm{\nabla_{\theta_t} L(\Psi(x_i;
\theta_t), y_i)}_2$. However, computing the norm of the gradient for each
sample is computationally intensive. \citet{alain2015variance} use a
distributed cluster of workers and constrain their models to fully connected
networks while \citet{zhao2015stochastic} only consider convex problems and
sample according to the Lipschitz constant of the loss of each sample, which is
an upper bound of the gradient norm.

To mitigate the computational requirement, we propose to use the loss itself as
the importance metric instead of the gradient norm. Although providing the
network with a larger number of confusing examples makes intuitive sense, we
show that the loss can be used to create a tighter upper bound to the gradient
norm than plain uniform sampling using the following theorem.

\begin{theorem} \label{the:loss_upper_bound}
Let $G_i = \norm{\nabla_{\theta_t} L(\Psi(x_i; \theta_t), y_i)}$ and $M = \max
G_i$. There exist $K > 0$ and $C < M$ such that
\begin{equation}
\frac{1}{K} L(\Psi(x_i; \theta_t), y_i) + C \geq G_i \quad \forall i
\end{equation}
\end{theorem}

The above theorem is derived using the fact that the ordering of samples
according to the loss is reflected by their ordering according to the gradient
norm. A complete proof of the theorem is provided in the supplementary
material. Despite the fact that theorem \ref{the:loss_upper_bound} only proves
that sampling using the loss is an improvement compared to uniform sampling, in
our experiments we show that it exhibits similar variance reducing properties
to sampling according to the gradient norm thus resulting in convergence speed
improvement compared to uniform sampling.

To sum up, we propose the following importance sampling scheme that
creates an unbiased estimator of the gradient vector of Batch Gradient
Descent with lower variance:
\begin{align}
    p_i      & \propto L(\Psi(x_i; \theta_t), y_i) \\
    \alpha_i & = \frac{1}{N p_i}.
\end{align}

\subsection{Relation to Max Loss Minimization} \label{sec:max_loss_theory}

Minimizing the average loss over the training set does not necessarily result
in the best model for classification. \citet{shalev2016minimizing} argue that
minimizing the maximum loss can lead to better generalization performance,
especially if there exist a few ``rare'' samples in the training set.
In this section, we show that introducing a minor bias inducing modification in
the sample weights $\alpha_i$, we are able to focus on the high-loss samples
with a variable intensity up to the point of minimizing the maximum loss.

Instead of choosing the sample weights $\alpha_i$ such that we
get an unbiased estimator of the gradient, we define them according to
\begin{equation} \label{eq:biased_weights}
\alpha_i = \frac{1}{N p_i^k}, \ \text{with} \ k \in (-\infty, 1].
\end{equation}
Using $L_i = L(\Psi(x_i; \theta), y_i)$ and $p_i \propto L_i$, we get
\begin{align}
\E[P]{\alpha_i \nabla_{\theta} L_i}
& = \sum_{i=1}^N p_i \alpha_i \nabla_{\theta} L_i \\
& = \sum_{i=1}^N \frac{p_i^{1-k}}{N} \nabla_{\theta} L_i \\
& \propto \sum_{i=1}^N \frac{ L_i^{1-k}}{N} \nabla_{\theta} L_i \\
& \propto \sum_{i=1}^N \frac{1}{N} \nabla_{\theta} L_i^{2-k}, \label{eq:soft_max_loss}
\end{align}
which is an unbiased estimator of the gradient of the original loss
function raised to the power $2-k \geq 1$. When $2-k \gg 1$ we essentially
minimize the maximum loss but as it will be analyzed in the experiments,
smaller values can be helpful both in increasing the convergence speed and
improving the generalization error.

\subsection{Approximate Importance Sampling} \label{sec:approx_is}

\begin{figure*}
\makebox[\textwidth][c]{
    \begin{subfigure}[t]{0.49\textwidth}
        \includegraphics[width=\linewidth]{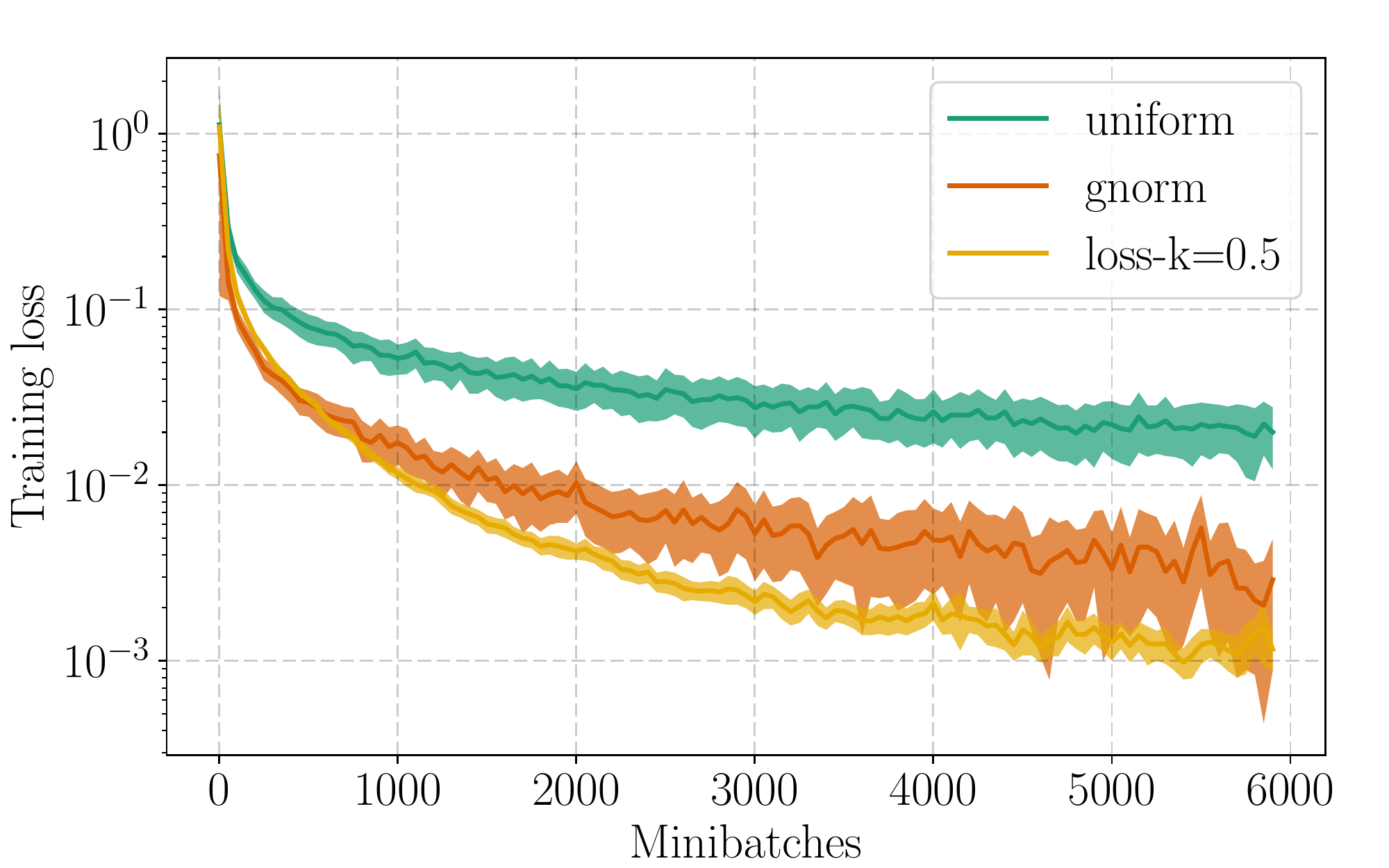}
        \caption{Variance reduction} \label{fig:loss_vs_gnorm}
    \end{subfigure}
    ~
    \begin{subfigure}[t]{0.49\textwidth}
        \includegraphics[width=\linewidth]{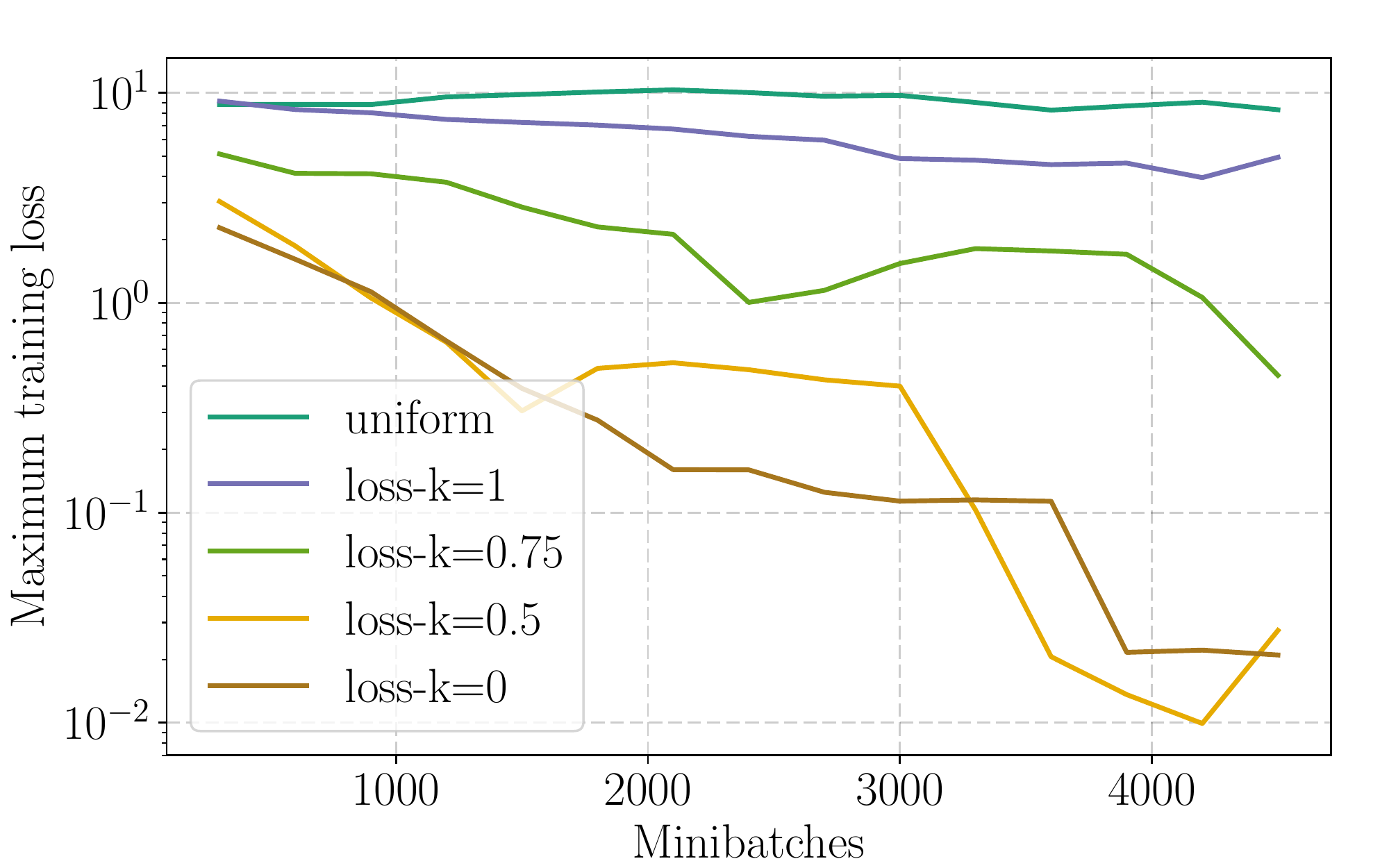}
        \caption{Effect of $k$} \label{fig:k_effects}
    \end{subfigure}
    }
    \caption{Averaged results of $10$ independent runs on MNIST. Solid lines in
    Figure~\ref{fig:loss_vs_gnorm} show the moving average of the training loss
    and shaded areas show the moving standard deviation.
    Figure~\ref{fig:k_effects} shows the 5 point moving average of the
    \emph{maximum} training loss for different values of the hyperparameter
    $k$. Smaller values of $k$ minimize the maximum loss in contrast to $k=1$,
    which corresponds to unbiased importance sampling, that behaves similar to
    uniform sampling.} \label{fig:mnist}
\end{figure*}

Although by using the loss instead of the gradient norm, we
simplify the problem and make it straightforward for use with Deep Learning
models, calculating the loss for a portion of the training set is still
prohibitively resource intensive. To alleviate this problem and make importance
sampling practical, we propose approximating the loss
with another model, which we train alongside our Deep Learning model.

Let $\mathcal{H}_t = \{(j, \tau, L_j^{\tau}) \mid j \in \{0, 1, \dots, N\},
\tau \leq t\}$ be the history of the losses where the triplet $(j, \tau,
L_j^{\tau})$ denotes the sample index $j$, the iteration index $\tau$ and the
value of the loss for sample $j$ at iteration $\tau$ (samples seen in the same
mini-batch appear in the history as triplets with the same $\tau$.).

Our goal is to learn a
model $M(x_i, y_i, \mathcal{H}_{t-1}) \approx L(\Psi(x_i; \theta_t), y_i)$ that
has negligible computational complexity compared to $\Psi(x_i; \theta_t)$. The
above formulation can be used to define any model, including approximations of
the original neural network $\Psi(\cdot)$. In order to create lightweight
models, that do not impact the performance of a single forward-backward pass (or
impact it minimally), we focus on models that use the class information and the
history of the losses. Specifically, we consider models that map the history
and the class to two separate representations that are then combined with a
simple linear projection.

To generate a representation for the history, we run an LSTM over the previous
losses of a sample and return the hidden state. The use of an LSTM allows the
history of other samples to influence the representation through the shared
weights. Let this mapping be represented by $M_h(j, \mathcal{H}_{t-1}; \pi_h)$
parameterized by $\pi_h$. Regarding the class mapping, we use a simple
embedding, namely we map each class to a specific vector in $\mathbb{R}^D$. Let
this mapping be $M_y(y_j; \pi_y)$  parameterized by $\pi_y$.

\begin{algorithm}
\caption{Approximate importance sampling} \label{alg:training}
\begin{algorithmic}[1]
    \State Assume inputs $\eta$, $\pi_0$, $\theta_0$, $k \in (-\infty, 1]$, $X = \{x_1,
           x_2, \dots, x_N\}$ and $Y = \{y_1, y_2, \dots, y_N\}$
    \State $t \gets 0$
    \Repeat
        \State $S \sim \text{Uniform}(1, N)$
            \Comment{Sample a portion of the dataset for further speedup}
        \State $p_i \propto M(i, y_i, \mathcal{H}_t) \, \forall i \in S$
        \State $s \sim \text{Multinomial}(P)$ \label{alg:sample}
        \State $\alpha \gets \frac{1}{N p_s^k}$
        \State $\theta_{t+1} \gets \theta_t - \eta \alpha \nabla_{\theta_t}
            L(\Psi(x_s ; \theta_t), y_s)$ \label{alg:sgd1}
        \State $\pi_{t+1} \gets \pi_t - \eta \nabla_{\pi_t}
            M(s, y_s, \mathcal{H}_t; \pi_t)$ \label{alg:sgd2}
        \State $\mathcal{H}_{t+1} \gets \mathcal{H}_t \cup
            \{(s, t, L(\Psi(x_s ; \theta_t), y_s))\}$
        \State $t \gets t + 1$
    \Until{convergence}
\end{algorithmic}
\end{algorithm}

Finally, we solve the following optimization problem and learn a function that
predicts the importance of each sample for the next training iteration.
\begin{equation} \label{eq:approx_training}
\begin{aligned}
M_i &= M\big( M_h(i, \mathcal{H}_{t-1}; \pi_h), M_y(y_i; \pi_y) ; \pi \big) \\
L_i &= L\big(\Psi(x_i; \theta_t), y_i\big) \\
\pi^*, \pi_h^*, \pi_y^* &= \argmin_{\pi, \pi_h, \pi_y} \frac{1}{N} \sum_{i=1}^N 
    (M_i - L_i)^2
\end{aligned}
\end{equation}

The precise training procedure is described in pseudocode in
algorithm~\ref{alg:training} where, for succinctness, we use $M(\cdot; \pi)$ to
denote the composition and the parameters of the models $M(\cdot)$,
$M_h(\cdot)$ and $M_y(\cdot)$.

\subsubsection{Smoothing} \label{sec:smoothing}

Modern Deep Learning models often contain stochastic layers, such as Dropout,
that given a set of constant parameters and a sample can result into vastly
different outputs for different runs. This fact, combined with the inevitable
approximation error of our importance model, can result into pathological cases
of samples being predicted to have a small importance (thus large weight
$\alpha_i$) but ending up having high loss.

To alleviate this problem we use \emph{additive smoothing} to influence the
sampling distribution towards uniform sampling. We observe experimentally, that
a good rule of thumb is to add a constant $c$ such that $c \leq \frac{1}{2N}
\sum_{i=1}^N L(\Psi(x_i; \theta_t), y_i)$ for all iterations $t$.

\section{Experiments}\label{sec:experiments}

\begin{figure*}
    \makebox[\textwidth][c]{
        \begin{subfigure}[t]{0.49\textwidth}
            \includegraphics[width=\linewidth]{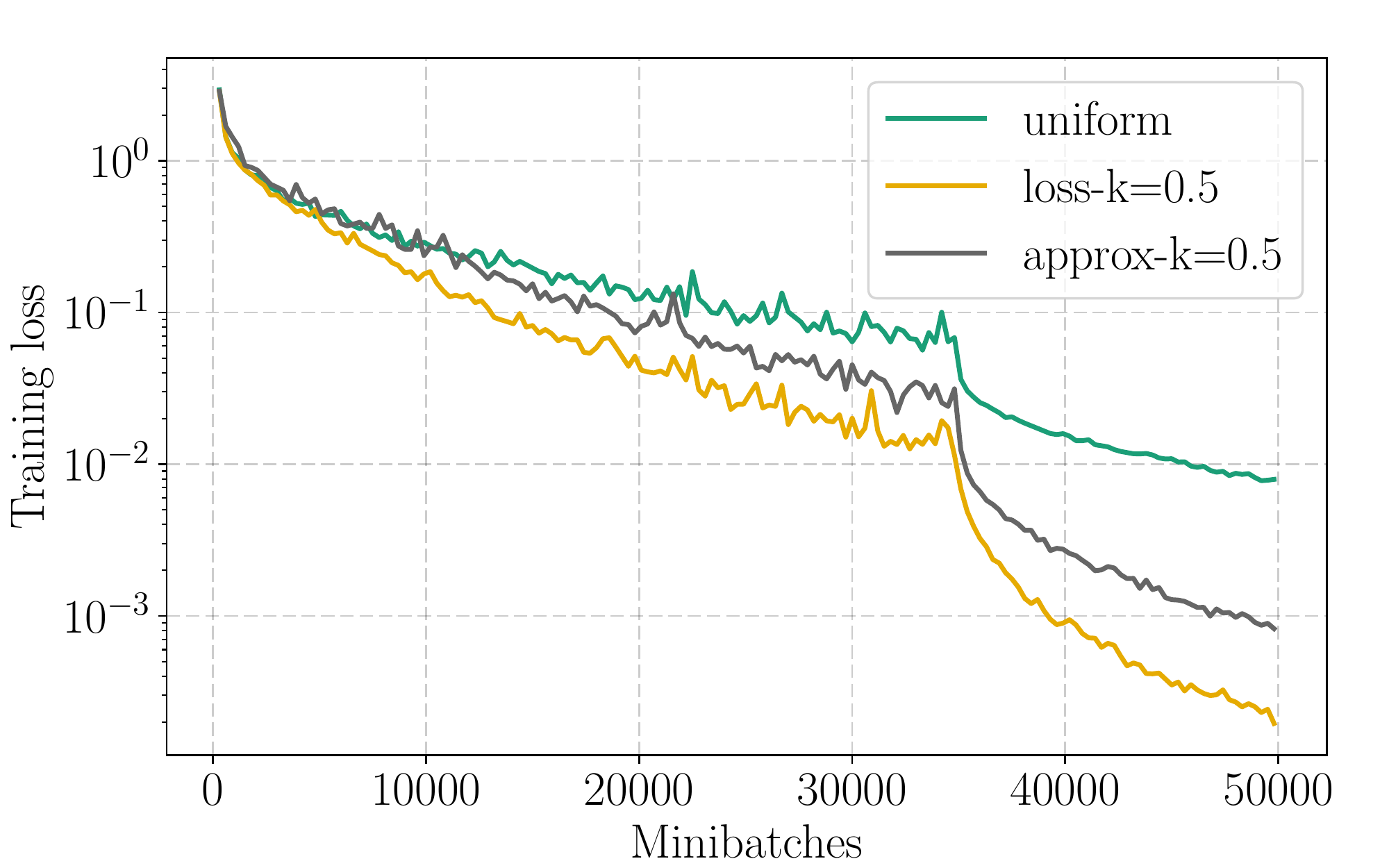}
            \caption{Speed-up per epochs} \label{fig:cifar_epochs}
        \end{subfigure}
        ~
        \begin{subfigure}[t]{0.49\textwidth}
            \includegraphics[width=\linewidth]{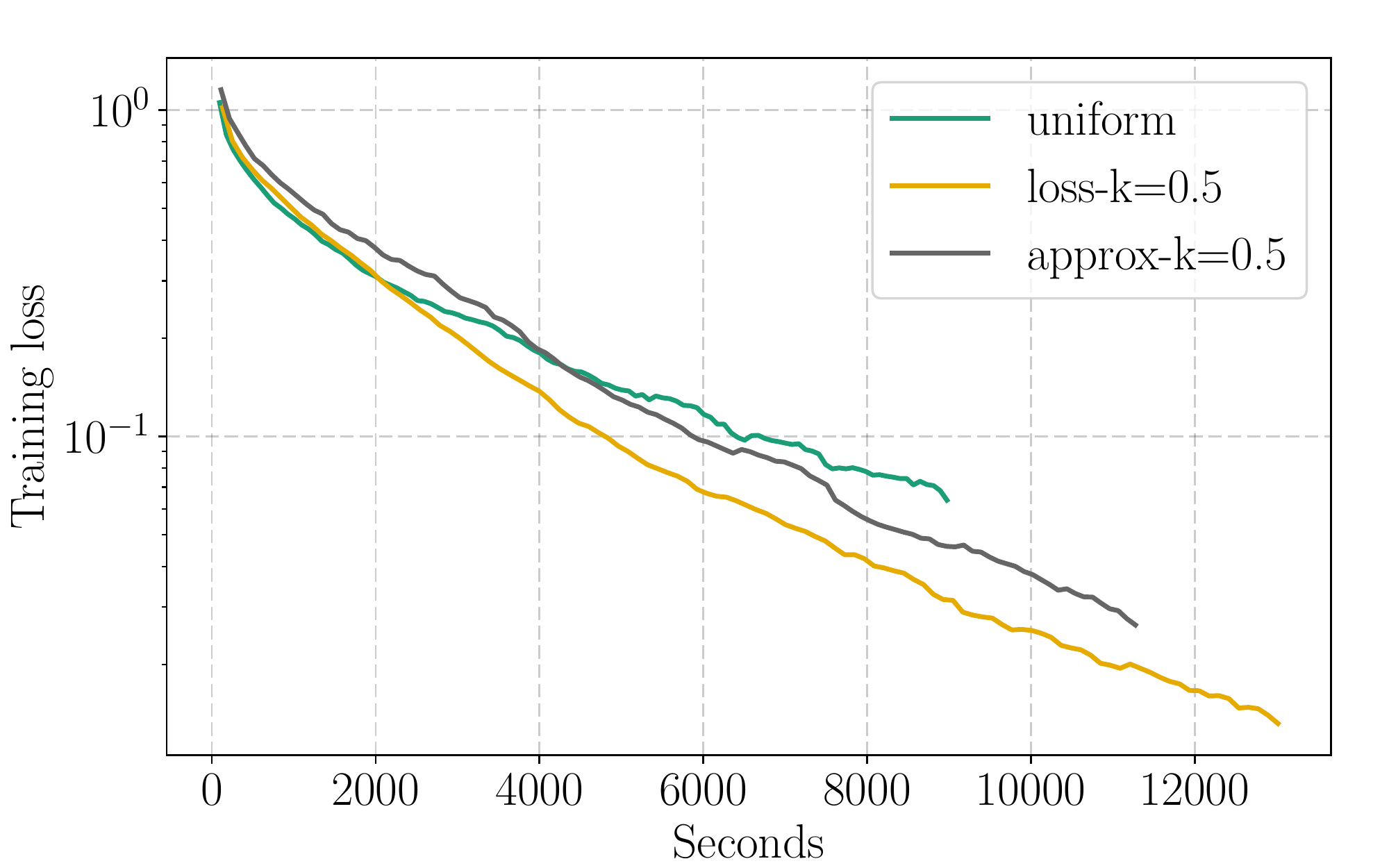}
            \caption{Wall-clock time Speed-up} \label{fig:cifar_seconds}
        \end{subfigure}
    }
    \caption{Training evolution results for CIFAR10 (average of $3$ runs).
    Figure~\ref{fig:cifar_epochs} depicts the speed-up achieved with importance
    sampling in terms of epochs while Figure~\ref{fig:cifar_seconds} shows the
    wall-clock time improvement achieved with importance sampling for the first
    $35,000$ iterations of training (before the learning rate reduction).
    Figure~\ref{fig:cifar_seconds} shows an 11-point moving average.}
    \label{fig:cifar}
\end{figure*}

In this section, we analyze experimentally the behaviour of the proposed
importance sampling schemes. In the first subsection, we show the variance
reduction followed by a comparison with the sampling using the gradient norm
and an analysis of the effects of the hyperparameter $k$ (from equation
\ref{eq:biased_weights}), which controls the bias. In the following
subsections, we present the attained improvements in terms of training time and
test error reduction compared to uniform sampling. For our experiments, we use
three well known benchmark datasets, MNIST~\cite{lecun1998mnist},
CIFAR10~\cite{krizhevsky2009learning} and Penn
Treebank~\cite{marcus1993building}.

We compare the proposed sampling strategies, \textbf{loss}, which uses the
actual model to calculate the importance and \textbf{approx}, which uses our
approximation defined in the previous sections to \textbf{uniform} sampling
which is our baseline and \textbf{gnorm} that uses the gradient norm as the
importance.

In particular, the \textbf{approx} is implemented using an LSTM with a hidden
state of size $32$. The input to the LSTM layer is at most the $10$ previously
observed loss values of a sample (features of one dimension). Regarding the
class label, it is initially projected in $\mathbb{R}^{32}$ and subsequently
concatenated with the hidden state of the LSTM. The resulting $64$ dimensional
feature is used to predict the loss with a simple linear layer.

In all the experiments, we deviate slightly from our theoretical analysis and
algorithm by sampling mini-batches instead of single samples in line
\ref{alg:sample} of Algorithm~\ref{alg:training}, and using the Adam
optimizer~\cite{kingma2014adam} instead of plain Stochastic Gradient Descent in
lines \ref{alg:sgd1} and \ref{alg:sgd2} of Algorithm~\ref{alg:training}.

Experiments were conducted using Keras \cite{chollet2015} with TensorFlow
\cite{abadi2016tensorflow}, and the code to reproduce the experiments will be
provided under an open source license when the paper will be published. When
wall clock time is reported, the experiment was run using an Nvidia K80 GPU and
the reported time is calculated by subtracting the timestamps before starting one
epoch and after finishing one; thus it includes the time needed to transfer
data between CPU and GPU memory.

\subsection{Behaviour of loss-based Importance Sampling}

In this section, we present an in-depth analysis of our \emph{biased importance
sampling}, by performing experiments on the well know hand-written digit
classification dataset, MNIST. Given the simplicity of the dataset, we choose a
down-sized variant of the VGG architecture~\cite{simonyan14c}, with two
convolutional layers with $32$ filters of size $3\times3$, a max-pooling layer
and a dropout layer with a dropout rate of $0.25$, without batch normalization,
followed by the same combination of layers with twice the number of filters.
Finally, we add two fully connected layers with $512$ neurons and dropout with
rate $0.5$, and a final classification layer. All the layers use the ReLU
activation function.

\subsubsection{Comparison with the gradient norm} The first experiment of this
section compares the loss to the gradient norm as importance metrics. We train
each network for $6,000$ iterations with a mini-batch size of $128$. For each
set of parameters we run $10$ independent runs with different random seeds and
report the average. For the pre-sampling of Algorithm~\ref{alg:training} we use
$2 B$ where $B$ is the size of the mini-batch, so in this case the importance
is computed for $256$ randomly selected samples per parameter update.

In Figure~\ref{fig:loss_vs_gnorm}, we observe that our importance sampling
scheme accelerates the training in terms of epochs by $6\times$. It is also
evident that the variance is reduced compared to uniform sampling. In this
experiment, we observed that sampling with the gradient norm is more than $100$
times slower, with respect to wall-clock time, than sampling using the loss.
Moreover, compared to the optimal sampling using the gradient norm, we show
that gradient norm is very sensitive to randomness introduced by the Dropout
layers and thus performs a bit worse than our proposed method. The above could
be one of the reasons why no other work has ever used the gradient norm to
train Deep Neural Network models with Dropout or Batch Normalization.

\subsubsection{Analysis of the sampling bias} 

Using the loss as the importance metric for sampling our mini-batches allows us
to efficiently minimize a surrogate to a soft max-loss, namely the original
loss raised to the power $2-k$ (see equation~\ref{eq:soft_max_loss}). We
theorize (based on \cite{shalev2016minimizing})  that minimizing such a loss
will be beneficial to the generalization ability of the trained model. Although
this will be empirically tested in the last section of the experiments, in this
section, we aim to  empirically validate our theory that using a small $k$
focuses on the max-loss of the dataset.

For this experiment, we train, again, each network for $6,000$ iterations with
$k \in \{0, 0.5, 0.75, 1\}$. Every $300$ iterations, we compute the loss for
each image in the training set. Figure~\ref{fig:k_effects} depicts the moving
average of the maximum loss for all the parameter combinations. We observe that
smaller values for $k$, indeed, minimize the maximum loss while $k=1$, which
corresponds to unbiased importance sampling, ignores the maximum loss in the
same way that uniform sampling does. In addition, using $k < 0$ has been
observed to result in noisy training, which can be explained by the fact that
we increase the gradients of the samples with large gradient norm, thus
resulting in very big steps per iteration. Due to the above, we use $k=0.5$ for
all subsequent experiments, which has been found to be a good compromise
between variance reduction and max-loss minimization.

\subsection{Convergence speed-up}

\begin{figure*}
    \makebox[\textwidth][c]{
        \begin{subfigure}[t]{0.49\textwidth}
            \includegraphics[width=\linewidth]{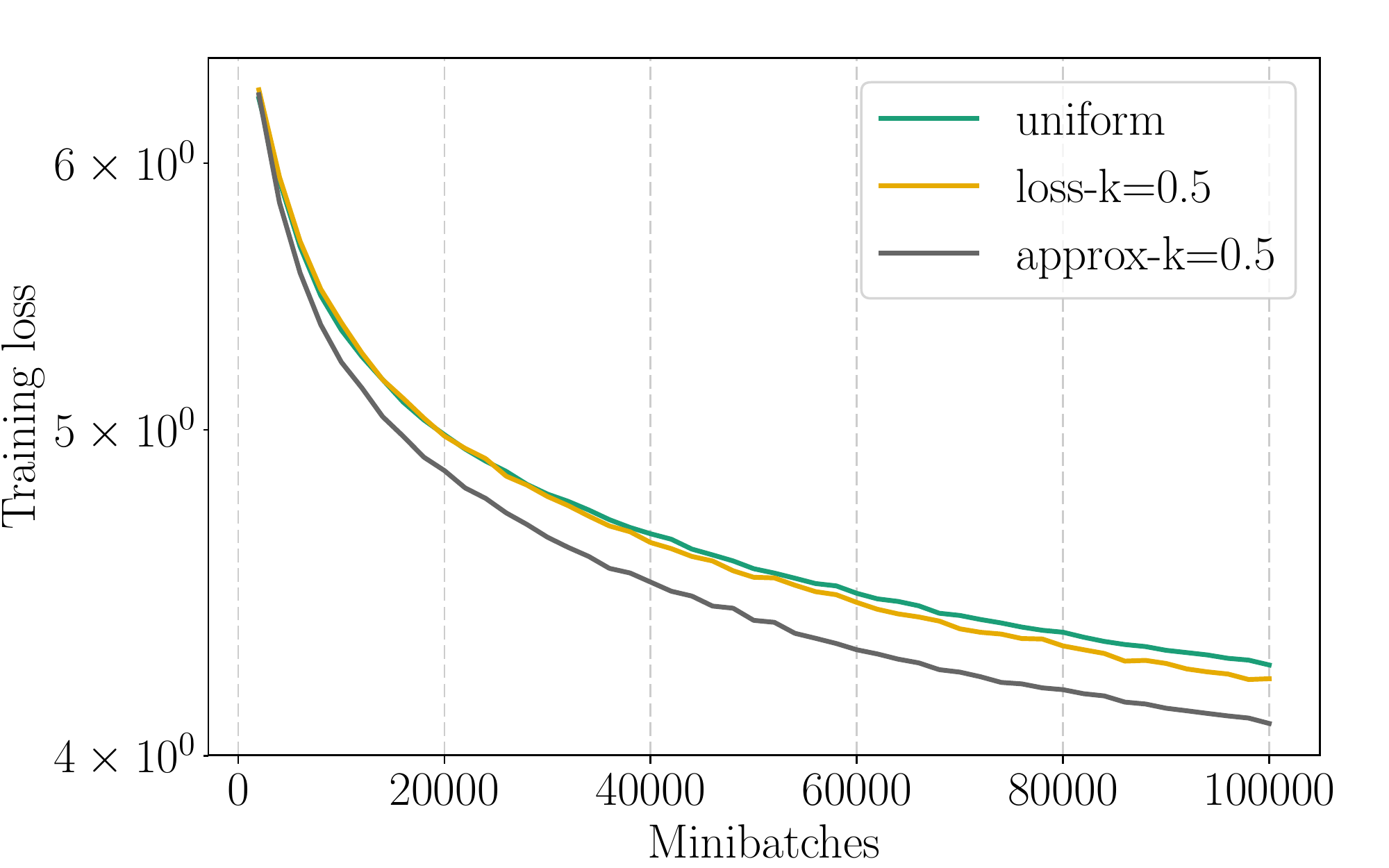}
            \caption{Speed-up per epochs} \label{fig:ptb_epochs}
        \end{subfigure}
        ~
        \begin{subfigure}[t]{0.49\textwidth}
            \includegraphics[width=\linewidth]{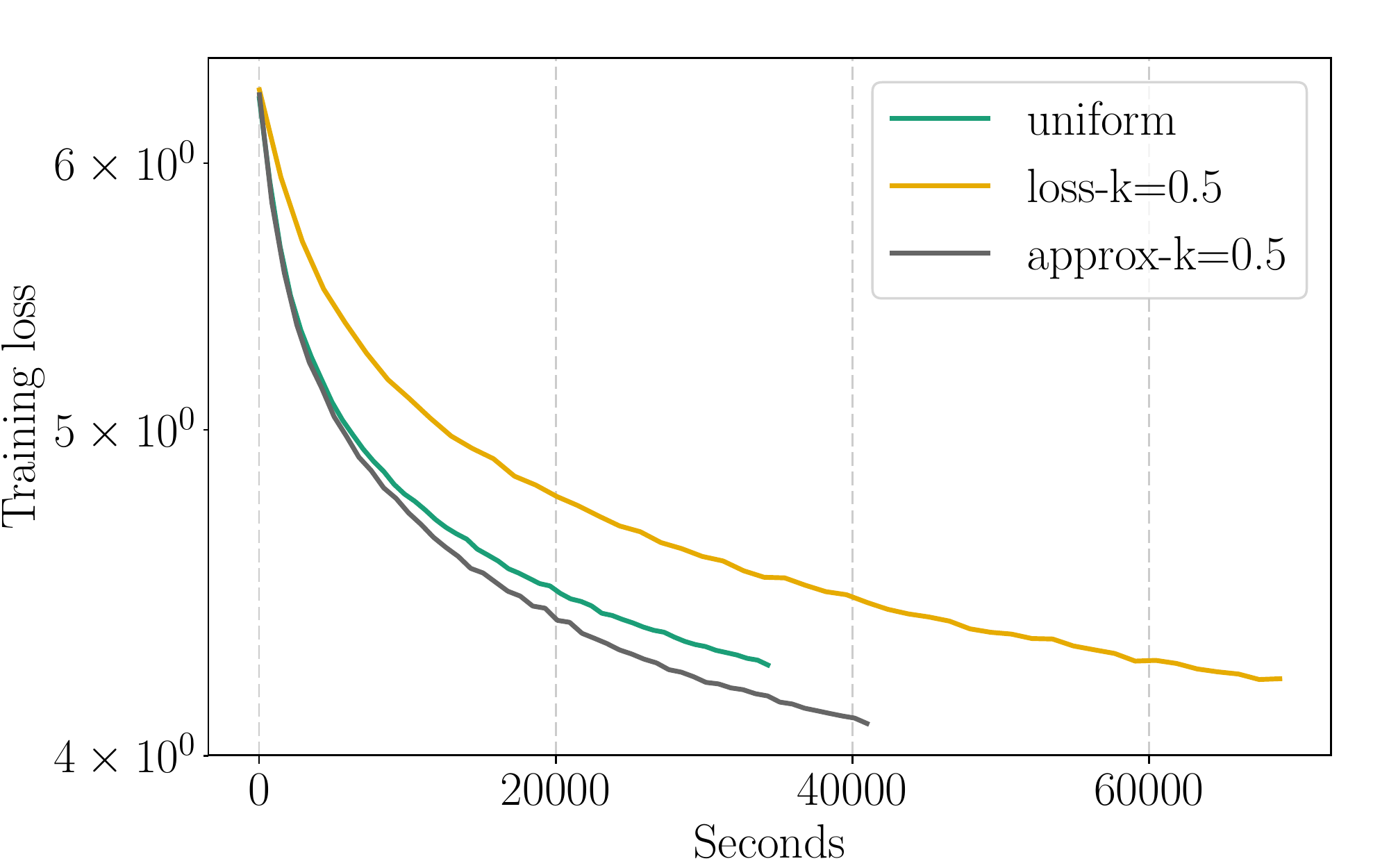}
            \caption{Wall-clock time Speed-up} \label{fig:ptb_seconds}
        \end{subfigure}
    }
    \caption{Training evolution results for Penn Treebank for $100,000$
    iterations (average of $3$ runs).  Figure~\ref{fig:ptb_epochs} depicts the
    speed-up achieved with importance sampling in terms of epochs while
    Figure~\ref{fig:ptb_seconds} shows the wall-clock time improvement.}
    \label{fig:ptb}
\end{figure*}

Our second series of experiments aim to show the convergence speed-up achieved
by sampling according to the loss (\textbf{loss}) and by our fast approximation
trained alongside the full model (\textbf{approx}). To that end we perform
experiments on CIFAR10~\cite{krizhevsky2009learning}, an image classification
dataset commonly used to evaluate new Deep Learning methods, and Penn
Treebank~\cite{marcus1993building}, a widely used word prediction dataset.

\subsubsection{Image classification}

For the CIFAR10 dataset, we developed a VGG-inspired network, with batch
normalization, which consists of three convolution-pooling blocks with $64$,
$128$ and $256$ filters, two fully connected layers of sizes $1024$ and $512$,
and a classification layer. Dropout is used in a similar manner as for the
MNIST network with rates $0.25$ after each pooling layer and $0.5$ between the
fully connected layers. The activation function in all layers is ReLU.

To compare the performance of the methods with respect to training epochs, each
network is trained for $50,000$ iterations using a batch size of $128$ samples.
After $35,000$ iterations we decrease the learning rate by a factor of $10$. We
perform minor data augmentation by creating $500,000$ images
generated by random flipping, horizontal and vertical shifting of the original
images. For the importance sampling strategies, we use smoothing as described
in the previous sections. In particular, the importance of each sample is
incremented by $\frac{1}{2} \bar{L}$, where $\bar{L}$ is the mean of the
training loss computed by the exponential moving average of the mini-batch
losses. Furthermore, we run each method with $3$ different random seeds and
report the mean. Similar to the previous experiment, we calculate the importance on a
uniformly sampled set twice the batch size and resample with importance (see
line \ref{alg:sample} Algorithm~\ref{alg:training}). The results, which are
depicted in Figure~\ref{fig:cifar_epochs}, show that importance sampling
accelerates the training significantly in terms of epochs, resulting in an
order of magnitude better training loss after $50,000$ iterations.

In order for our proposed method to be useful as is, we need to achieve
speed-up with respect to wall clock time as well as epochs. In
Figure~\ref{fig:cifar_seconds}, we plot the training loss with respect to the
seconds passed for the initial $35,000$ iterations of training for CIFAR10. The
experiment shows that our \textbf{approx} model is $\sim 10\%$ faster than
\textbf{loss} and $\sim 20\%$ slower than uniform per epoch. However, we
observe that \textbf{approx} reaches the final loss attained with uniform
sampling in $\sim 15\%$ less time while \textbf{loss} in $\sim 30\%$ less time
shaving off approximately $45$ minutes out of the $2.5$ hours of training.

One of the reasons that our approximation does not perform as well can be
explained by observing the first epochs in Figure~\ref{fig:cifar_epochs} or the
first seconds in Figure~\ref{fig:cifar_seconds}. Our approximation needs to
collect information on the evolution of the losses in the set $\mathcal{H}_t$
(see equation \ref{eq:approx_training}); thus it does not perform well in the
initial epochs. Moreover, the GPU handles very well the small $32\times32$
images thus the full loss incurs less than $50\%$ slowdown per epoch compared
to uniform sampling.

\subsubsection{Word prediction}

\begin{table*}
\caption{Experimental results across all datasets and methods. The results are
averaged over multiple runs ($10$ for MNIST and $3$ for CIFAR10 and Penn Treebank)
and we report the mean and the standard deviation. For the Penn Treebank we
report the perplexity in the validation set and for the rest the classification
error. The learning rate is chosen to maximize the performance of the uniform
sampling strategy for all cases. Runs labeled as {\bf loss} use the trained
model to predict the loss while {\bf approx} use a smaller model to approximate
the loss as defined in the previous sections. An epoch is defined as $300$
mini-batches for MNIST/CIFAR10, and $2000$ for Penn Treebank respectively.}
\label{tab:results}
\vskip 3mm
\makebox[\textwidth][c]{
\def\arraystretch{1.15}
\small
\begin{tabular}{|c|c|ccccc|}
\hline
& \multirow{2}{*}{Epoch} & \multicolumn{5}{|c|}{Method} \\
\cline{3-7}
& & Uniform & Loss k=1 & Loss k=0.5 & Approx k=1 & Approx k=0.5 \\
\hline
\multirow{3}{*}{MNIST}
    & 5  & 0.79\% {$\pm$ 0.11} & 0.73\% {$\pm$ 0.11} &
           0.56\% {$\pm$ 0.04} & 0.83\% {$\pm$ 0.06} & 0.65\% {$\pm$ 0.05} \\
    & 10 & 0.64\% {$\pm$ 0.10} & 0.56\% {$\pm$ 0.07} &
           0.52\% {$\pm$ 0.05} & 0.60\% {$\pm$ 0.06} & 0.61\% {$\pm$ 0.04} \\
    & 20 & 0.56\% {$\pm$ 0.07} & 0.53\% {$\pm$ 0.08} &
           0.47\% {$\pm$ 0.04} & 0.53\% {$\pm$ 0.06} & 0.54\% {$\pm$ 0.06} \\
\hline
\multirow{3}{*}{CIFAR10}
    & 50  & 12.33\% {$\pm$ 0.39} & 10.68\% {$\pm$ 0.65} &
            9.58\% {$\pm$ 0.16} & 11.54\% {$\pm$ 0.55} & 10.78\% {$\pm$ 0.09} \\
    & 100 & 10.19\% {$\pm$ 0.24} & 9.74\% {$\pm$ 0.44} &
            9.20\% {$\pm$ 0.55} & 9.71\% {$\pm$ 0.14} & 9.90\% {$\pm$ 0.47} \\
    & 150 & 7.97\% {$\pm$ 0.10} & 7.77\% {$\pm$ 0.14} &
            7.44\% {$\pm$ 0.14} & 7.61\% {$\pm$ 0.07} & 7.64\% {$\pm$ 0.27} \\
\hline
\multirow{3}{*}{PTB}
    & 10 & 184.5 {$\pm$ 1.56} & 178.3 {$\pm$ 2.47} &
           187.2 {$\pm$ 1.61} & 179.4 {$\pm$ 6.20} & 180.0 {$\pm$ 4.78} \\
    & 30 & 138.6 {$\pm$ 0.60} & 137.8 {$\pm$ 1.53} &
           139.0 {$\pm$ 0.72} & 136.2 {$\pm$ 2.98} & 134.3 {$\pm$ 2.19}  \\
    & 50 & 130.3 {$\pm$ 0.63} & 129.9 {$\pm$ 1.27} &
           130.5 {$\pm$ 0.21} & 128.2 {$\pm$ 2.02} & 127.4 {$\pm$ 1.45} \\
\hline
\end{tabular}
\hspace*{0.5em}
}
\end{table*}

To assess the generality of our method, we  conducted experiments on a language
modeling task. We used the Penn Treebank language dataset, as preprocessed by
\citet{mikolov2011empirical}\footnote{\url{http://www.fit.vutbr.cz/~imikolov/rnnlm/}},
and a recurrent neural network as the language model. Initially, we split the
dataset into sentences and add an ``End of sentence'' token (resulting in a
vocabulary of 10,001 words). For each word we use the previous 20 words, if
available, as context for the neural network. Our language model is similar to
the small LSTM used by \citet{zaremba2014recurrent} with 256 units, a word
embedding in $\mathbb{R}^{64}$ and dropout with rate 0.5.

In terms of training the language models, we use a batch size of 128 words and
train each network for $100,000$ iterations. For the importance sampling
strategies, we use constant smoothing instead of adaptive as in the CIFAR10
experiment. To choose the smoothing constant, we experiment with the values
$\{0.5, 1, 2.5\}$ and choose 0.5 because it performs better in terms of
variance minimization during training. Finally, we run each method with 3
different random seeds and report the mean. Similar to the two previous
experiments, we pre-sample uniformly 256 words, namely twice the mini-batch
size.

The results of the language modeling experiment are presented in Figures
\ref{fig:ptb_epochs} and \ref{fig:ptb_seconds}. We observe (from
Figure~\ref{fig:ptb_epochs}), that using importance sampling indeed improves
the convergence speed in terms of epochs. However, using the full \textbf{loss}
performs marginally better than \textbf{uniform} random sampling. The gradient
variance is reduced more using the actual model in this case as well, which
leads us to the following two explanations:
\begin{enumerate}
    \item A small amount of noise is beneficial to the training procedure (we
    could add noise to the \textbf{loss} method by increasing the smoothing
    parameter)
    \item The slower changes in sample importance between iterations (side
    effect of using the history for our approximation) are beneficial to
    importance sampling
\end{enumerate}
Moreover, it is important to note in Figure~\ref{fig:ptb_seconds} that even if
using the \textbf{loss} performed better with respect to epochs than
\textbf{approx}, it requires twice as much time to train a given number of
epochs compared to \textbf{uniform}. On the other hand, \textbf{approx} is
merely $\sim 10\%$ slower per epoch, resulting $\mathbf{20\%}$ less time to
achieve the final training loss, shaving almost \textbf{$2$ hours} from the
$9.5$ hours required for \textbf{uniform} sampling to perform $100,000$
iterations.

We have also attempted to compare our method with the method presented by
\citet{loshchilov2015online}. Although we managed to find a set of
hyperparameters that achieve comparable speed-ups for CIFAR10, their method
completely fails to converge for Penn Treebank, for the parameters we tested.
This behaviour, reflects the experiments of the original paper and can possibly
be explained by the ad-hoc aggressive weighting of samples that they employ.

\subsection{Generalization improvement}

In this experimental section we analyze the performance of the previously
trained models on unseen data. Our goal is to show that importance sampling
does not cause overfitting or affect the performance on the test set in any
negative way.

Towards this end, we report the classification error on the test set for
experiments conducted using the MNIST and CIFAR10 datasets while for the
language modeling task we report the perplexity on the validation set.  The
results, for all the experiments, are summarized in Table~\ref{tab:results},
where the values are depicted for the early stages, medium stages and final
stages of training.

We observe that the test error is the same or better than using uniform
sampling in all cases. More importantly, Table~\ref{tab:results} shows that our
added bias is indeed helpful in reducing the generalization error, especially
in the first stages of training. For MNIST we observe a $0.2\%$ reduction in
error at the 5-th epoch when using $k=0.5$ and for CIFAR10 approximately $1\%$
at the 30-th epoch. For Penn Treebank the case is not as clear although we
observe a reduction of $2$ at the 30-th epoch when compared to the
\textbf{approx} method with $k=1$.

Finally, we have empirically observed that using $k=0.5$ results in a robust
procedure that requires less smoothing and is affected less by noisy importance
estimation.

\section{Conclusion}

In this paper, we show theoretically and empirically that the loss can be used
as an importance metric to accelerate the training of any Deep Neural Network
architecture including recurrent ones. In particular, we propose a biased
importance sampling scheme (with and without a lightweight approximation), that
results in significant speed-up both in terms of wall-clock time and epochs.
Using this sampling scheme, we train Deep Neural Networks for image
classification and word prediction tasks $20\%$ to $30\%$ faster than using
uniform sampling with Adam.

Two important points remain to be investigated thoroughly. The first is the
construction of a more principled method to approximate the loss of the full
model. Modern neural networks contain hundreds of layers; thus it is hard to
imagine that there exists no approximation to their output that can fully take
advantage of the cost/accuracy trade-off.

The second point is the exploitation of the reduced variance of the gradient
estimator by some other part of the stochastic optimization procedure.
Importance sampling, for instance, could be used to perform stochastic line
search or improve the robustness of stochastic Quasi-Newton methods.

\bibliographystyle{aaai}
\bibliography{references}

\clearpage

\standalonetitle{Appendix}
\appendix
\section{Justification for sampling with the loss}

The goal of this analysis is to justify the use of the loss as the importance
metric instead of the gradient norm and provide additional evidence (besides
the experiments in the paper) that it is an improvement over uniform sampling.

Initially, we show that sampling with the loss is better at minimizing an upper
bound to the variance of the gradients than uniform sampling.  Subsequently, we
provide additional empirical evidence that the loss is a surrogate for the
gradient norm by computing the exact gradient norm for a limited number of
samples and comparing it to the loss.

\subsection{Theoretical justification} \label{sec:theoretical}

Initially, we show in lemma \ref{lem:loss_order} that the most common loss
functions for classification and regression define the same ordering as their
gradient norm. Subsequently, we use this derivation together with an upper
bound to the variance of the gradients and show that sampling according to the
loss reduces this upper bound compared to uniform sampling.

Firstly, we prove two lemmas that will be later used in the analysis.

\begin{lemma} \label{lem:convex_order}
Let $f(x): \mathbb{R} \to \mathbb{R}$ be a strictly convex monotonically
decreasing function then
\begin{equation}
f(x_1) > f(x_2) \iff \abs{\diff{f}{x_1}} > \abs{\diff{f}{x_2}} \ \forall \ x_1, x_2 \in
\mathbb{R}.
\end{equation}
\end{lemma}

\begin{proof}
By the definition of $f(x)$ we have
\begin{align}
\ddiff{f}{x} > 0 \ \forall \ x \\
\diff{f}{x} \leq 0 \ \forall \ x,
\end{align}
which means that $\diff{f}{x}$ is monotonically increasing and non-positive. In
that case we have
\begin{align}
x_1 < x_2 & \iff f(x_1) > f(x_2) \\
x_1 < x_2 & \iff \diff{f}{x_1} < \diff{f}{x_2} \iff \abs{\diff{f}{x_1}} > \abs{\diff{f}{x_2}}
\end{align}
therefore proving the lemma.
\end{proof}

\begin{lemma} \label{lem:loss_order}
Let $L(\psi) : D \to \mathbb{R}$ be either the negative log
likelihood or the squared error loss function defined respectively as
\begin{equation}
\begin{aligned}
L_1(\psi) & = - y^T log(\psi) & y \in \{0, 1\}^d \ \text{s.t.} y^T y = 1 \\
    & & D=[0, 1]^d \ \text{s.t.} \norm{\psi}_1 = 1 \\
L_2(\psi) & = \norm{y - \psi}_2^2 & y \in \mathbb{R}^d \quad
    D=\mathbb{R}^d
\end{aligned}
\end{equation}
where $y$ is the target vector. Then
\begin{equation}
L(\psi_1) > L(\psi_2) \iff \norm{\nabla_{\psi} L(\psi_1)} > \norm{\nabla_{\psi} L(\psi_2)}
\end{equation}
\end{lemma}

\begin{proof}
In the case of the squared error loss we have
\begin{equation}
\norm{\nabla_{\psi} L(\psi)}_2^2
     = \norm{-2 (y - \psi)}_2^2
     = 4 L(\psi),
\end{equation}
thus proving the lemma.

For the log likelihood loss we can use the fact that only one dimension of $y$
can be non-zero and prove it using lemma \ref{lem:convex_order} because
$f(x) = -log(x)$ is a strictly convex monotonically decreasing function.
\end{proof}

\subsubsection{Main analysis} \label{sec:theoretical_main}

We can now prove theorem~\ref{the:loss_upper_bound} which is also written below
for completeness.

\begin{theorem} \label{the:loss_upper_bound}
Let $G_i = \norm{\nabla_{\theta_t} L(\Psi(x_i; \theta_t), y_i)}$ and $M = \max
G_i$. There exist $K > 0$ and $C < M$ such that
\begin{equation}
\frac{1}{K} L(\Psi(x_i; \theta_t), y_i) + C \geq G_i \quad \forall i
\end{equation}
\end{theorem}

\begin{proof}
The goal of importance sampling is to minimize
\begin{equation}
\Tr{\V{\nabla_{\theta} L(\Psi(x_i; \theta), y_i)}} = \E{\norm{\nabla_{\theta}
    L(\Psi(x_i; \theta), y_i)}_2^2}.
\end{equation}
To perform importance sampling, we sample according to the distribution $P$
with probabilities $p_i$ and use per sample weights $\alpha_i = \frac{1}{N
p_i}$ in order to have an unbiased estimator of the gradients. Consequently,
the variance of the gradients is
\begin{align}
\E[P]{\norm{\alpha_i \nabla_{\theta} L(\Psi(x_i; \theta), y_i)}_2^2} &=\\
\sum_{i=1}^N p_i \alpha_i^2 \norm{\nabla_{\theta} L(\Psi(x_i; \theta), y_i)}_2^2 &=\\
\sum_{i=1}^N \frac{1}{N p_i} \frac{1}{N} \norm{\nabla_{\theta} L(\Psi(x_i; \theta), y_i)}_2^2 &=\\
\sum_{i=1}^N \alpha_i \frac{1}{N} \norm{\nabla_{\theta} L(\Psi(x_i; \theta), y_i)}_2^2.
\end{align}

Assuming that the neural network is Lipschitz continuous (assumption that holds
when the weights are not infinite) with constant $K$, we derive the following
upper bound to the variance
\begin{align}
&\E[P]{\norm{\alpha_i \nabla_{\theta} L(\Psi(x_i; \theta), y_i)}_2^2} \leq\\
&    \sum_{i=1}^N \alpha_i \frac{1}{N} \norm{\nabla_{\theta} \Psi(x_i; \theta)}_2^2
        \norm{\nabla_{\Psi(x_i; \theta)} L(\Psi(x_i; \theta), y_i)}_2^2 \leq\\
&    K^2 \sum_{i=1}^N \alpha_i \frac{1}{N}
        \norm{\nabla_{\Psi(x_i; \theta)} L(\Psi(x_i; \theta), y_i)}_2^2. \label{eq:gnorm_upper_bound}
\end{align}

Since we have a finite set of samples, there exists a constant $C$ such that
\begin{equation}
\begin{aligned}
L(\Psi(x_i; \theta), y_i) + C &\geq
    \norm{\nabla_{\Psi(x_i; \theta)} L(\Psi(x_i; \theta), y_i)} \\
& \forall \ i \in \{1, 2, \dots, N\}. \label{eq:loss_upper_bound}
\end{aligned}
\end{equation}
However, using lemma \ref{lem:loss_order} we know that this upper bound is
better than uniform because $L(\Psi(x_i; \theta), y_i)$ and
$\norm{\nabla_{\Psi(x_i; \theta)} L(\Psi(x_i; \theta), y_i)}$ grow and shrink
in tandem. In particular the following equation holds, where $M = \max
\norm{\nabla_{\Psi(x_i; \theta)} L(\Psi(x_i; \theta), y_i)}$ and $C < M$
\begin{equation} \label{eq:loss_upper_bound2}
\begin{aligned}
& L(\Psi(x_i; \theta), y_i) + C - \norm{\nabla_{\Psi(x_i; \theta)}
  L(\Psi(x_i; \theta), y_i)} < \\
& M - \norm{\nabla_{\Psi(x_i; \theta)}
  L(\Psi(x_i; \theta), y_i)} \quad \forall i.
\end{aligned}
\end{equation}

Using equations \ref{eq:gnorm_upper_bound}, \ref{eq:loss_upper_bound} and
\ref{eq:loss_upper_bound2} and replacing the constants in
\ref{eq:gnorm_upper_bound} with $K$ we derive the original claim concluding the
proof.
\end{proof}

\subsection{Empirical justification} \label{sec:empirical}

In this section, we provide empirical evidence regarding the use of the loss
instead of the gradient norm as the importance metric. Specifically, we conduct
experiments computing the exact gradient norm and the loss value for the first
20,000 samples during training. The gradient norm is normalized in each
mini-batch to account for the changes in the norm of the weights.

Subsequently, we plot the loss sorted by the gradient norm. If there exists $C$
such that $L(\Psi(x_i; \theta), y_i) = C \norm{\nabla_{\theta} L(\Psi(x_i;
\theta), y_i)}$ we should see approximately a line. In case of order
preservation we should see a monotonically increasing function.

\begin{figure*}
\makebox[\textwidth][c]{
    \begin{subfigure}[b]{0.36\textwidth}
        \includegraphics[width=\textwidth]{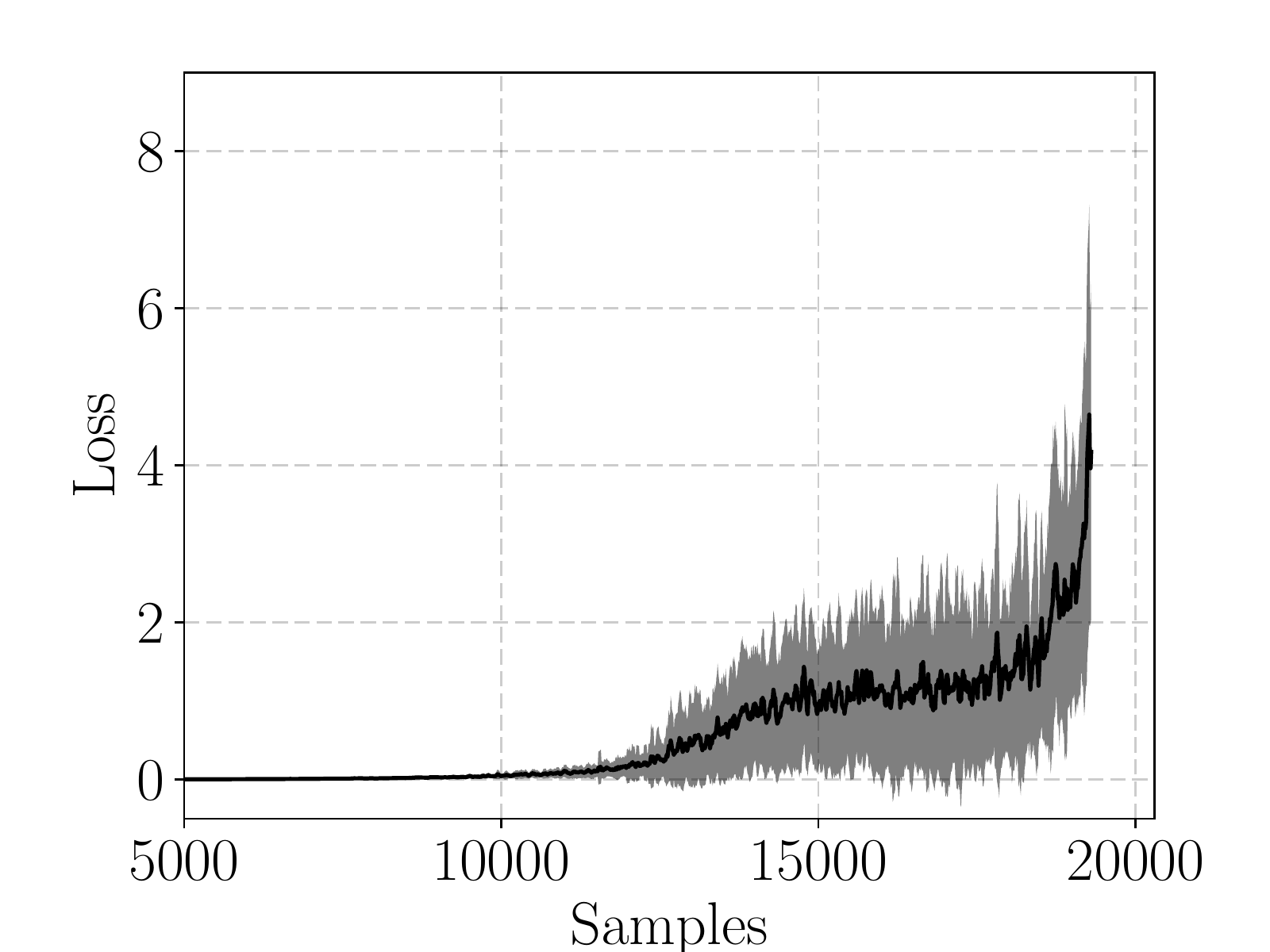}
        \caption{MNIST}
        \label{fig:mnist_gnorm_loss}
    \end{subfigure}
    \begin{subfigure}[b]{0.36\textwidth}
        \includegraphics[width=\textwidth]{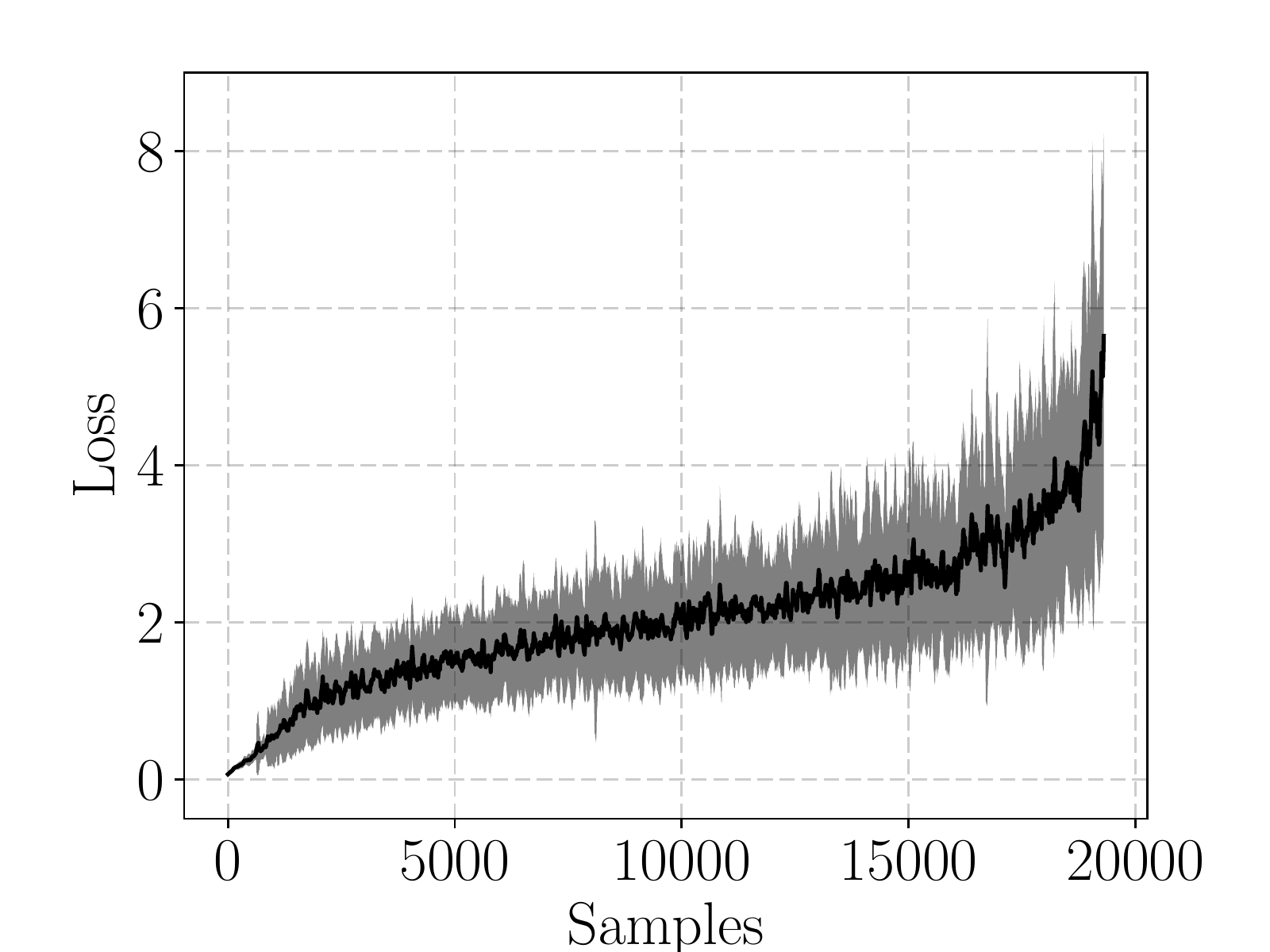}
        \caption{CIFAR10}
        \label{fig:cifar_gnorm_loss}
    \end{subfigure}
    \begin{subfigure}[b]{0.36\textwidth}
        \includegraphics[width=\textwidth]{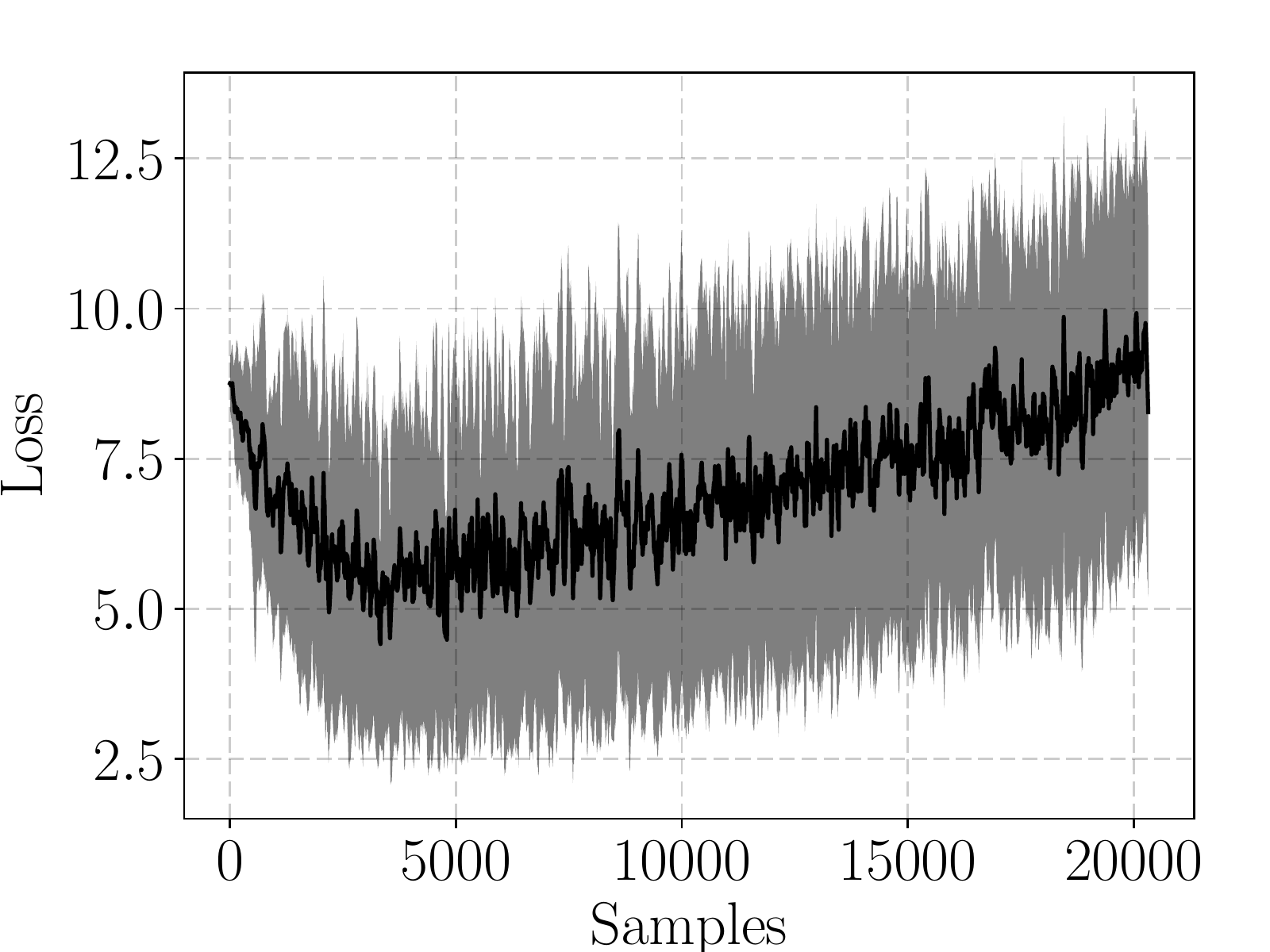}
        \caption{Penn Treebank}
        \label{fig:ptb_gnorm_loss}
    \end{subfigure}
    \hspace*{1em}
    }
    \caption{Loss values sorted by gradient norm. The solid line is a moving
    average ($50$ samples window) and the shaded area denotes one standard
    deviation.} \label{fig:gnorm_loss}
\end{figure*}

In figure \ref{fig:gnorm_loss}, we observe a correlation between the gradient
norm and the loss. In all cases, samples with high gradient norm also have high
loss. In the Penn Treebank dataset, (Figure~\ref{fig:ptb_gnorm_loss}), there
exist some samples with high loss but very low gradient norm. This can be
explained because LSTMs use the $\text{tanh}(\cdot)$ activation function which
can have very low gradient but incorrect output. We note that this cannot hurt
performance, it just means that we waste some CPU/GPU cycles on samples that are
incorrectly classified but will not affect the parameters heavily.

\vfill\eject

\section{Loss tracking by our approximation}

In this section, we present empirical evidence that the proposed approximate
model predicts the loss with reasonable accuracy. In order to describe the
experiment, we introduce the following notation. Let $\hat{L}_i$ be the loss
value computed for the $i$-th sample and used as importance in the sampling
procedure. In addition, let $L_i$ be the loss computed in the forward pass of
the optimization procedure, with the same parameters and for the same sample.
As mentioned in the paper, $L_i$ and $\hat{L}_i$ can differ even if we use the
full network to compute $\hat{L}_i$, due to stochasticity introduced by Dropout
layers and Batch Normalization.

To compute how well $\hat{L}_i$ ``tracks'' $L_i$ we solve the following
least squares problem for every minibatch. The value of the coefficient $a$
shows how well $\hat{L}_i$ approximates $L_i$.

\begin{equation} \label{eq:prediction_quality}
a^*, b^* = \argmin_{a, b} \sum_{i \in B} \left(a \hat{L}_i + b - L_i\right)^2
\end{equation}

In Figure~\ref{fig:tracking}, we present the evolution of the coefficient $a$
for the experiments in CIFAR10 and Penn Treebank. Regarding CIFAR10, we observe
that the loss is really hard to track, even using the full Neural Network.
However, the approximation presents a positive correlation with $L_i$ which
appears to be enough for reducing the variance of the gradients. Moreover,
after $15,000$ mini-batches the approximation ``tracks'' the loss with the same
accuracy as the full network. On the other hand, in the Penn Treebank experiment,
there seems to be less noise and $L_i$ is predicted accurately by both methods.
Once again, we observe, that as the training progresses our approximation seems
to converge towards the same quality of predictions as using the full network.

\begin{figure*}
\makebox[\textwidth][c]{
    \begin{subfigure}[b]{0.49\textwidth}
        \includegraphics[width=\textwidth]{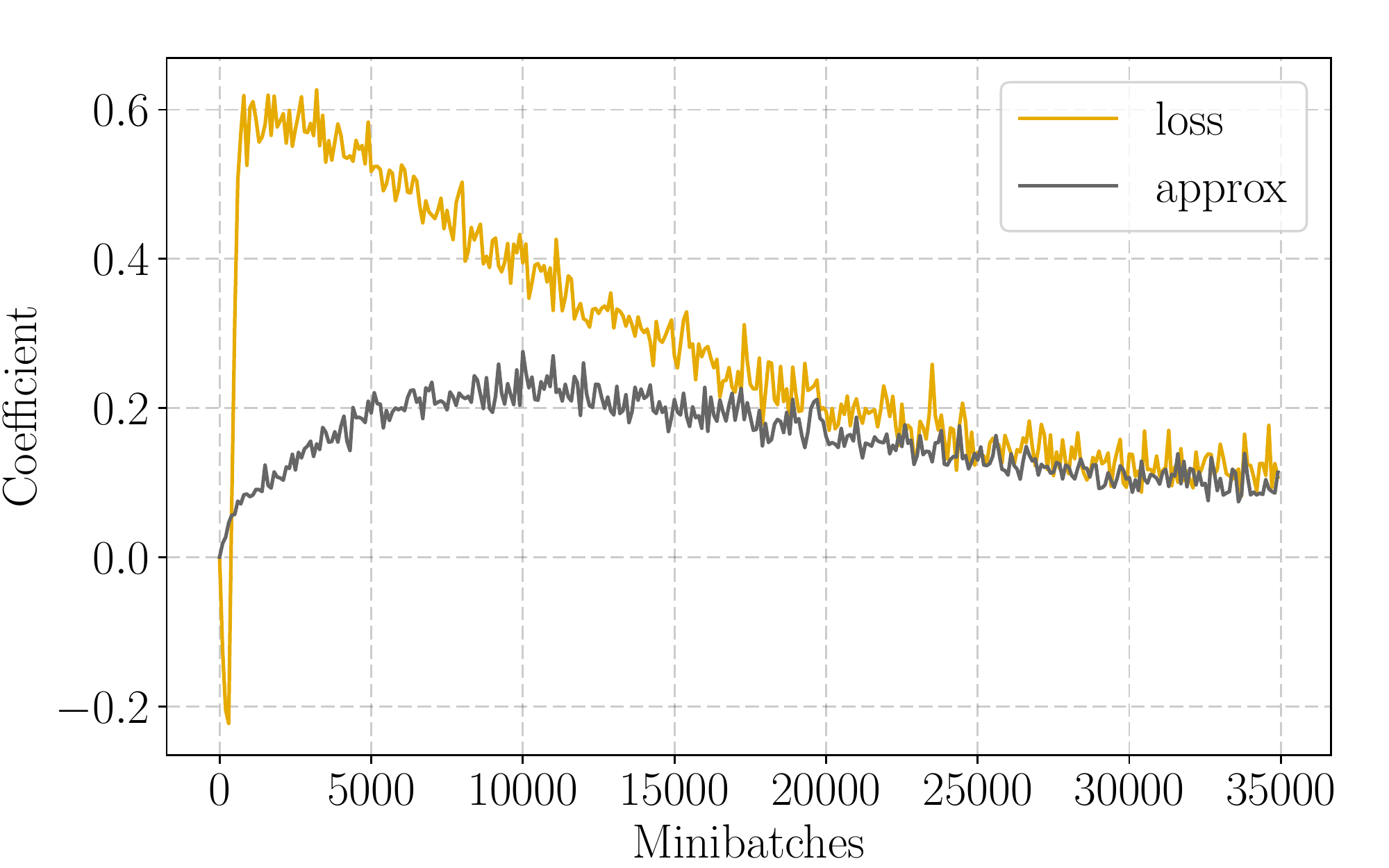}
        \caption{CIFAR10}
        \label{fig:tracking_cifar}
    \end{subfigure}
    \begin{subfigure}[b]{0.49\textwidth}
        \includegraphics[width=\textwidth]{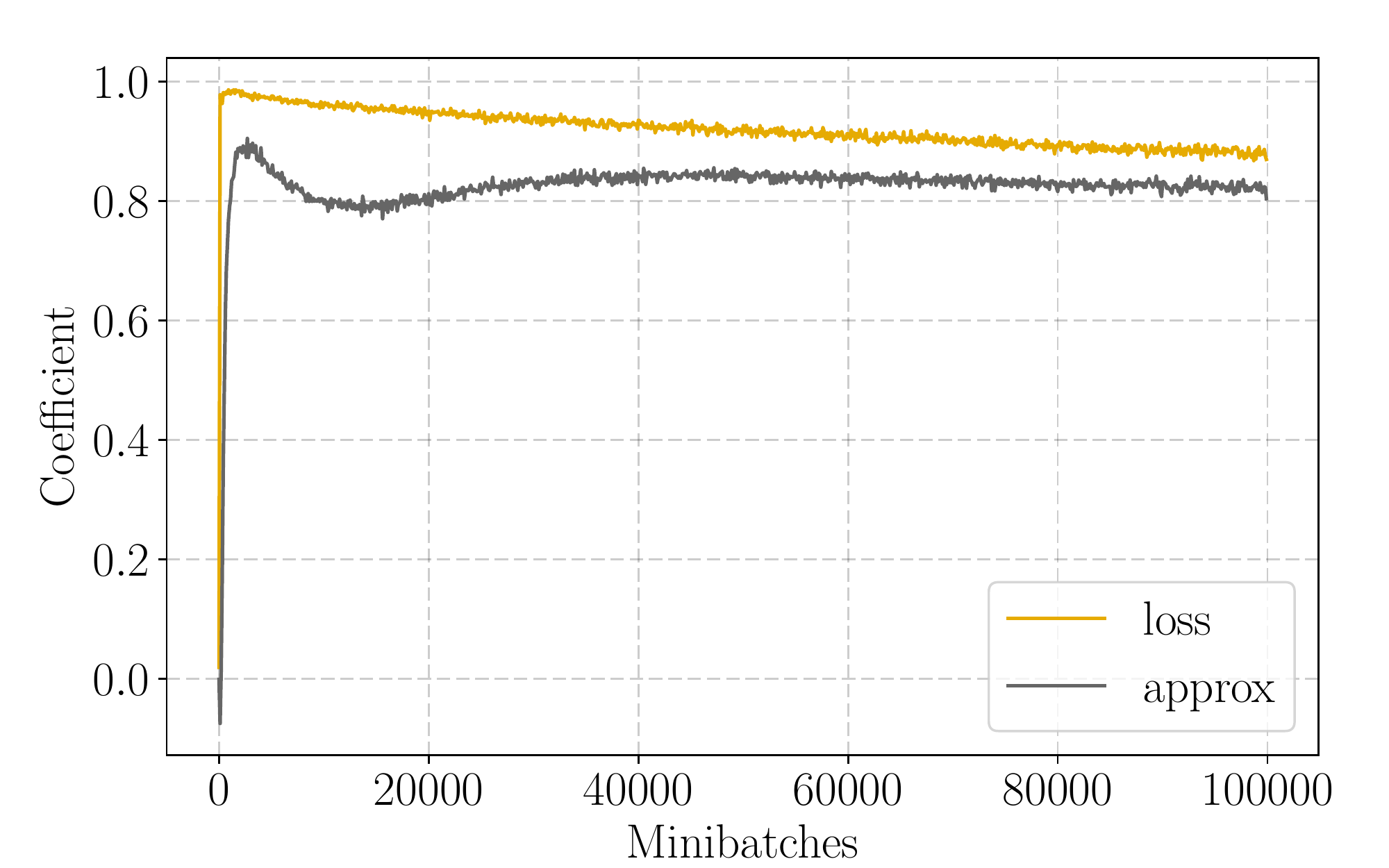}
        \caption{Penn Treebank}
        \label{fig:tracking_ptb}
    \end{subfigure}
    \hspace*{1em}
    }
    \caption{The coefficient $a$ from equation~\ref{eq:prediction_quality} is
    plotted for \textbf{loss} and \textbf{approx}. A coefficient value of $1$
    means that the loss is tracked perfectly while $0$ means almost uniform
    sampling. Due to randomness introduced by layers such as Dropout and Batch
    Normalization, even using the full network (\textbf{loss}) does not mean
    perfect prediction.}
    \label{fig:tracking}
\end{figure*}

\checknbdrafts

\end{document}